\documentclass{article}

\usepackage{graphicx}
\usepackage[T1]{fontenc}
\usepackage{graphics}
\usepackage{algorithm,amsmath, amssymb,amsthm}
\usepackage{algorithmic}
\usepackage{tikz}
\usetikzlibrary{positioning}
\usepackage{epstopdf,epsfig}
\newtheorem{lemma}{Lemma}

\newtheorem{assumption}{Assumption}
\newtheorem{remark}{Remark}

\newtheorem{example}{Example}
\definecolor{deanPURPLE}{rgb}{.3,0,.5}
\definecolor{jrBLUE}{rgb}{.2, .2, .8}
\definecolor{gray}{rgb}{.8,.8,.8}

\newcommand{\comment}[1]{}

\title{Optimal Weighting of Multi-View Data with Low Dimensional Hidden States}
\author{Yichao Lu \and Dean P. Foster}

\begin{document}

\tikzstyle{state}=[shape=circle,draw=blue!50,fill=blue!20]
\tikzstyle{observation}=[shape=circle,draw=blue!50,fill=red!20,thick]
\tikzstyle{duration}=[shape=rectangle,draw=blue!50,fill=red!20,thick]
\tikzstyle{lightedge}=[<-,thick,red]
\tikzstyle{mainstate}=[state,thick]
\tikzstyle{mainedge}=[<-,thick]

\maketitle
\begin{abstract}
In Natural Language Processing (NLP) tasks, data often has the following two properties: 
First, data can be chopped into multi-views which has been successfully used for dimension reduction purposes. For example, in topic classification, every paper can be chopped into the title, the main text and the references. However, it is common that some of the views are less noisier than other views for supervised learning problems. Second, unlabeled data are easy to obtain while labeled data are relatively rare. For example, articles occurred on New York Times in recent 10 years are easy to grab but having them classified as 'Politics', 'Finance' or 'Sports' need human labor. Hence less noisy features are preferred before running supervised learning methods. In this paper we propose an unsupervised algorithm which optimally weights features from different views when these views are generated from a low dimensional hidden state, which occurs in widely used models like Mixture Gaussian Model, Hidden Markov Model (HMM) and Latent Dirichlet Allocation (LDA).
\end{abstract}
\section{introduction}
In areas like Natural Language Processing, data often have multi-view and high dimension. Recently, CCA \cite{hotelling35} has been applied to the multi-view setting as a unsupervised dimension reduction method in \cite{foster07}\cite{kakade07}\cite{ando05} with performance guarantee if the data is generated under certain structure. In \cite{foster07}, they assume the high dimensional multi-view data is generated independently conditioning on a low dimensional hidden state (the model structure will be illustrated later in detail). Under this assumption, the low dimensional features provided by CCA won't lose any useful information compared with the original high dimensional features when applied to linear regression. Also, \cite{dhillon11} has applied this CCA method to generate a low dimensional vector representation of words which works well in a lot of NLP tasks.\\

The reason for CCA to work well is that the low dimensional hidden state (throughout the paper we'll use $k$ to denote the dimension of hidden state) contains most information for the supervised tasks and by doing CCA, we are able to generate $k$ dimensional estimate of the hidden state from each view as mentioned by \cite{bach05}, or more precisely, we can find all $k$ directions in the high dimensional space of each view that have non-zero correlation with the hidden state via CCA.\\

Only two views are enough to implement the CCA algorithms above (see \cite{foster07} for detailed introduction about CCA). Despite it's power in dimension reduction, CCA with two views is still not optimal in the sense that it ends up with a hidden state estimator from each view but it's impossible to tell which view is better by only looking at the two views. Here's an cute example:\\

\begin{example}
$h_0\sim N(0,1)$ be the hidden state. Conditioning on the hidden state, two views are generated independently with $v_1|h_0\sim N(h_0,0.1)$ and $v_2|h_0\sim N(h_0,10)$, Clearly $v_1$ is way better than $v_2$ if we want to estimate the hidden state since it's less noisier. However, since the only data we have are the two views, we can't do anything to figure out which view is more helpful in estimating the hidden state.
\end{example}
Similar situation happens in \cite{dhillon11} where the they have three views (the previous context, the current word, the latter context) and end up with three hidden state estimators.\\

This problem can be solved if we have three or more views. Actually, recent results have shown that more delicate problems can be solved if three or more views are available. \cite{hsu08} and \cite{song10} shows that we are able to compute sequential probability and conditional probability of an HMM with simple empirical statistics calculated from three consecutive observations. \cite{anim12} and \cite{DBLP12} proved that we are able to recover the emission matrix of mixture models with spectral methods when three different views of data are available. In this paper, we propose an algorithm where the hidden state estimators come from the three views are optimally combined  to get a cleaner estimator of the hidden state in the sense that all other directions in the space are uncorrelated with the hidden state.\\

The paper is organized as follows: In section 2 a formal mathematical statement of the problem and a short proof of the two view dimension reduction algorithm will be given as a warm up. In section 3 the three views algorithm will be stated and proofs are given. In section 4 experiments on simulated data are performed to illustrate the correctness and effectiveness of the three views algorithm. Section 5 is a short summary.
\\
\section{Preliminary}
\subsection{Model Set Up}
In the multi-view problem, we have several views $X=(X^1,X^2,..X^{n_0})$ of the input data where $X^i$ are $d_i\times 1$ random vectors and a target variable Y which need to be predicted. Take NLP problems as an example, each view $X^i$ can be the words in each paragraph of an article while $Y$ can be the topic. Or as mentioned in \cite{dhillon11} \cite{dhillon12}, $X^1$ is the previous context of a word, $X^2$ is the current word, $X^3$ is the latter context and $Y$ is properties of the current word. One key structure of our model which connects the response $Y$ and the multi-view features $X$ is the hidden state:\\
\begin{assumption}
(Conditional Independence Assumption) Conditioning on a $k$ dimensional hidden state $H$ ($k \ll d_i$ for all $i$), the one dimensional response $Y$, and the three views $X^1$, $X^2$, $X^3$ are independent (since our algorithm needs only three views, from now on we are going to assume $X$ has three views).
\end{assumption}
\begin{figure}[htbp]
\begin{center}
\begin{tikzpicture}[]

\node[mainstate] (H) at (5,2.7) {$\;\;H\;\;$};
\node[mainstate] (Y) at (5,4.8) {$\;\;Y\;\;$}
edge[mainedge](H);
\node[mainstate] (X1) at (2,1) {$\;X^1\;$}
	edge[mainedge](H);
\node[mainstate] (X2) at (5,1) {$\;X^2\;$}
	edge[mainedge](H);
\node[mainstate] (X3) at (8,1) {$\;X^3\;$}
	edge[mainedge](H);
\end{tikzpicture}
\end{center}
\caption{The model structure: Conditioning on the hidden state, three views $X^1$, $X^2$, $X^3$ and the response $Y$ are independent} 
\end{figure}
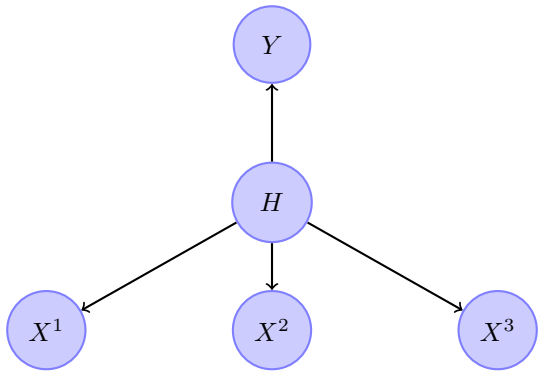
Moreover, in order CCA works,  we need assumption about the structure of the covariance matrix between each pair of views.\\

\begin{assumption}
(Linearity Assumption) $\mathbb{E}[Y|H],  \mathbb{E}[X^i|H]$ are all linear in $H$, i.e.  $\mathbb{E}[X^i|H]=M_{i}H$ and $\mathbb{E}[Y|H]=M_YH$ for some $d_i\times k$ matrix $M_i$ and $1\times k$ matrix $M_Y$.
\end{assumption}

\begin{assumption}
(Full Rank Assumption) The matrices $M_i, i=1,2,3$ have rank $k$.
\end{assumption}

A lot of models fall into this category. For example, Hidden Markov Model (HMM) which is widely used in NLP \cite{rabiner90}. Figure \ref{f2} shows an HMM with of length 3. Let the transition matrix be $T$ and the observation matrix be $O$, take $H_1$ as the hidden state, then $\mathbb{E}[X^1|H_1]=OH_1$, $\mathbb{E}[X^2|H_1]=OTH_1$, $\mathbb{E}[X^3|H_1]=OT^2H_1$ which are all linear in $H_1$. The Latent Dirichlet Allocation (LDA) model in \cite{blei03} \cite{anim12} and multi-view Gaussian Model like in \cite{bach05} \cite{DBLP12} also satisfies our assumptions.\\

\begin{figure}[htbp]
\begin{center}
\begin{tikzpicture}[]

\node[mainstate] (H1) at (2,2.8) {$\;H_1\;$};
	
\node[mainstate] (H2) at (5,2.8) {$\;H_2\;$}
	edge[mainedge](H1);
\node[mainstate] (H3) at (8,2.8) {$\;H_3\;$}
	edge[mainedge](H2);
\node[mainstate] (X1) at (2,1) {$\;X^1\;$}
	edge[mainedge](H1);
\node[mainstate] (X2) at (5,1) {$\;X^2\;$}
	edge[mainedge](H2);
\node[mainstate] (X3) at (8,1) {$\;X^3\;$}
	edge[mainedge](H3);
\end{tikzpicture}
\end{center}
\caption{HMM of length three satisfies assumption 1,2} 
\label{f2}
\end{figure}
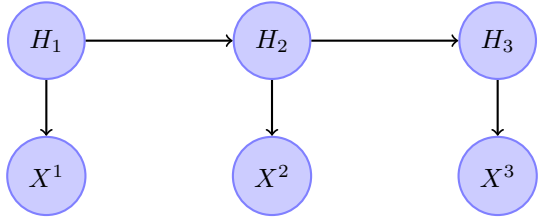
\subsection{Dimension Reduction with Two views}
In practice, $X^i$ are often high dimensional. For example, if $X^i$ are words in English, the dimension of the views are the size of the vocabulary. Another important issue is that in a lot of learning tasks, labelled data is rare while unlabeled data is common. In \cite{dhillon11} \cite{dhillon12} it's easy to get word with its context from the internet, while getting words labeled as 'plants' or 'animals' need a lot human effort. These observations lead to unsupervised dimension reduction algorithms which is illustrated by detail in \cite{foster07}. Here we briefly go though the two view case as a warm up for the three view situation. For simplicity, assume $H, X^1,X^2$ has $0$ mean and identity variances, since we can always whiten the views and the hidden state.\\

Let $\Sigma_{a,b}$ denote the covariance matrix of vector $a,b$ and $\Sigma_{i,Y}$ denote the covariance of $X^i$ and $Y$ (so the integer $i$ refers to the $i^{th}$ view). A straightforward conclusion following Assumption 1,2 is (lemma 7 in \cite{foster07}):
\begin{lemma}
\begin{eqnarray}
\Sigma_{1,2} &=&\mathbb{E}[\mathbb{E}[X^1{X^2}^T|H]]= M_1\mathbb{E}[HH^T]M_2^T=M_1M_2^T\\
\Sigma_{1,Y} &=&\mathbb{E}[\mathbb{E}[X^1{Y}^T|H]]= M_1\mathbb{E}[HH^T]M_Y^T=M_1M_Y^T
\end{eqnarray}
\end{lemma}
Since both views have identity variances, CCA between $X^1$ and $X^2$ reduce to only an Singular Value Decomposition(SVD) of the covariance matrix(see \cite{foster07} and \cite{hotelling35} for introduction about CCA).\\

Let  $\Sigma_{1,2}=UDV^T$ be the SVD of the covariance (since the variances are identity, it's also correlation) matrix. Let $U_{1:k}$, $V_{1:k}$ be the first $k$ columns of $U, V$. By assumption 3,
\begin{equation}
\Sigma_{1,2}=M_1M_2^T=U_{1:k}DV_{1:k}^T
\end{equation}
 By definition of CCA, the Canonical Variables are $X'^1=U^TX^1$ and $X'^2=V^TX^2$. \cite{foster07}  (theorem 3) claims that it suffices to pick the top $k$ canonical variables from each view, i.e. $X_{1:k}'^1=U_{1:k}^TX^1$ and $X_{1:k}'^2=V_{1:k}^TX^2$ as features if we predict $Y$ with linear regression.\\

The reason of their claim lies in two aspects:\\
First, the covariance between $X_{k+1:d_1}'^1$, the feature in view 1 we throw away and $Y$ is 
\begin{eqnarray}
\Sigma_{X_{k+1:d_1}'^1,Y}&=&U_{k+1:d_1}^TM_1M_Y^T
\end{eqnarray}
From equation (3),  the range of $M_1$ is the same as the range of $U_{1:k}$, hence columns of $U_{k+1:d_1}$ are orthogonal to columns of $M_1$. Together with (4), $\Sigma_{X_{k+1:d_1}'^1,Y}=0$. Similarly, $\Sigma_{X_{k+1:d_2}'^2,Y}=0$. In other words, the directions(or features) we dropped with CCA are uncorrelated with our target variable $Y$.\\
Second, we have the following lemma for linear regression:

\begin{lemma}
We have two group of features $(Z_1,Z_2)$, and want to predict $Y$ linearly with $(Z_1,Z_2)$. Suppose
the covariance matrices satisfy 
\begin{eqnarray*}
 \Sigma_{Y,Z_2}&=&0\\
 \Sigma_{Z_1,Z_2}&=&0
\end{eqnarray*}
Then the optimal linear predictor (in terms of the square loss) with $Z_1$ is the same as the optimal linear predictor with $(Z_1,Z_2)$.

\end{lemma}
\begin{proof}
Consider the Hilbert space of random variables where covariance is the inner product. The optimal linear predictor with $Z_1,Z_2$ is the projection of $Y$ onto the linear span of them. Our assumption means $Y$ perpendicular to span of $Z_2$ ($Y$ has zero covariance with $Z_2$ and covariance is the inner product), span of $Z_1$ perpendicular to span of $Z_2$, so the projection of $Y$ onto span of $Z_1,Z_2$ is the same as to the projection of $Y$ onto $Z_1$.
\end{proof}

Let $Z_1=(X'^1_{1:k},X'^2_{1:k})$ and $Z_2=(X'^1_{k+1:d_1},X'^2_{k+1:d_2})$, this partition satisfies lemma 2. Therefore the optimal linear predictor with the low dimensional feature $(X'^1_{1:k},X'^2_{1:k})$ will be the same as $X^1,X^2$, or in other words, we get dimension reduction from $d_1+d_2$ to $2k$ for free.\\

\begin{remark}
After doing CCA, we obtain one $k$ dimensional feature from each view, which can be regarded as estimators of the $k$ dimensional hidden state. In order to estimate some feature $Y$ which are independent of these views conditioning on hidden state, one can first estimate the hidden state via CCA(unsupervised), then predict $Y$ with the hidden state estimators(supervised). The key property of CCA is the features throw away are uncorrelated with the $Y$, so it's reasonable to expect the CCA method to work well with other linear learning methods.
\end{remark}

\section{Optimal Weighting via Three Views}
As introduced in previous section, the two view CCA helps reducing dimension of the views to $k$, the dimension of hidden state. But one drawback of the two view CCA is we get one low dimensional estimator of the hidden state from each view, which may not be equally informative as illustrated in example 1. For instance, the abstract, main content and references can all help classify the topic of a paper, but are not equally informative. The main contribution of this paper is we find a way to optimally combine the estimators of the hidden state from each view to get a new hidden state  estimator if three or more views are available.\\

Here is the precise statement:\\
Assume we have three $k$ dimensional views $X^1,X^2,X^3$(since we can reduce the dimension of each view to $k$ with the CCA) and $Y$ generated by a $k$ dimensional hidden state and satisfy assumption 1,2,3. Use $X=(X^1;X^2;X^3)$ to denote the catenation of the three views (so $X$ is a $3k\times 1$ vector). Our goal is to look for a $3k*k$ matrix $U_1$ such that  the optimal linear predictor (in terms of square loss) with the new $k$ dimensional feature $X^*=U_1^TX$ is the same as the optimal linear predictor with the $3k$ dimensional feature $X$. In other words, $U_1$ optimally combines the hidden state estimators from each view. Still assume everything is mean $0$ and the hidden state $H$ has identity variance.\\

The following lemma proves the existence of the optimal $k$ dimensional feature $X^*$:
\begin{lemma}
 There exist a $k$ dimensional subspace in the linear span of $X^1,X^2,X^3$ (which is $3k$ dimensional) such that the optimal predictor with this subspace is the same as the optimal predictor with the whole space.
\end{lemma}
\begin{proof}
Do a Canonical Correlation Analysis between random vectors $H$ and $X$, Let $X'_{1:k}$ denote the first $k$ canonical components of $X$, $X'_{k+1:3k}$ be the rest. Since $H$ is only $k$ dimensional, by the definition of CCA, $\Sigma_{X'_{k+1:3k},H}=0$ and $\Sigma_{X'_{1:k},X'_{k+1,3k}}=0$.\\
By assumption 2, $\mathbb{E}[X'_{K+1:3K}|H]=M_4H$ for some $2k*k$ matrix $M_4$. Since 
\begin{equation}
\Sigma_{X'_{k+1:3k},H}=\mathbb{E}[\mathbb{E}[X'_{k+1:3k}H^T|H]]=M_4\mathbb{E}[HH^T]=M_4I=0
\end{equation}
We know $M_4=0$. Lemma 1 implies $\Sigma_{X'_{k+1:3k},Y}=M_4M_Y^T=0$. Apply lemma 2, The optimal linear predictor with $X'$ (the same as optimal linear predictor with $X$) is the same as the optimal linear predictor with $X^*=X'_{1:k}$.
\end{proof}

Our algorithm find the above optimal subspace in a relatively indirect way. In order to illustrate the rationale behind the algorithm, it's helpful to dig a little bit into the CCA proof of lemma 3.\\

Let the rotation matrix on $X$ given by the above CCA be $U_0=(U_1,U_2)$, and $X'_{1:k}=U_1^TX$, $X'_{k+1:3k}=U_2^TX$. Let $Q=\Sigma_{X,X}^{\frac{1}{2}}$, $Q^{-1}$ can be used to whitten $X$ to have identity covariance. Let $X''=Q^{-1}X$, and $\Sigma_{X'',H}$ has the full SVD:
\begin{equation*}
\Sigma_{X'',H}=PDV_0^T
\end{equation*}
Since $X''$ and $H$ all has identity covariance, the above SVD actually gives the CCA rotation for random vector $X''$ and $H$, i.e $P^TX''$ are the canonical variables. Moreover, since $X''=Q^{-1}X$, we know $U_0=Q^{-1}P$ is the CCA rotation matrix for $X$. Let $P=(P_1, P_2)$ where $P_1$ denotes the first $k$ columns and $P_2$ denotes the last $2k$ columns, then 
\begin{eqnarray}
U_1&=&Q^{-1}P_1\\
U_2&=&Q^{-1}P_2
\end{eqnarray}
Our goal is to look for $U_1$, then we can get the optimal subspace by $X^*=U_1^TX$. The trick for the algorithm is, we first estimate the column space of $U_2$, which is relatively easy, then we can find  the column space of $P_2$ based on (7) since $Q$, as the square root of the covariance of $X$ is easy to estimate. By property of SVD, $P_1 \perp P_2$ (means the column spaces of the two matrices are perpendicular), so we can reconstruct the column space of  $P_1$ based on $P_2$ easily (note that $U_1$ is not perpendicular to $U_2$). Finally ,we can find column space of  $U_1$ with $P_1$ and $Q$ by (6).\\

Based on the above argument, it suffices to find the column space of $U_2$. We need the following lemma:
\begin{lemma}
Let $a\in R^{3k\times 1}$ be a direction in $3k$ dimensional space. If for any $b\in R^{k\times 1}$, $Cov(a^TX, b^TH)=0$, $a$ lies in the column space of $U_2$
\end{lemma}
\begin{proof}
Let $a=c+d$ where $c$ is in the column space of $U_1$ and $d$ is in column space of $U_2$ (since $U_1,U_2$ span the whole space and have no intersection except 0, this decomposition of $a$ is unique). It suffices to show $c=0$. Note that
\begin{eqnarray*}
Cov(a^TX,b^TH)&=&Cov(c^TX,b^TH)+Cov(d^TX,b^TH)\\
&=&Cov(c^TX,b^TH)
\end{eqnarray*}
since $d$ is in column space of $U_2$. Let $U_1=(u_1,u_2,u_3,..u_k)$, since $c$ lies in column space of $U_1$, $c=\sum_{i=1}^k\alpha_iu_i$. Pick $b$ to be any canonical directions of $H$, i.e any column of $V_0$, by the assumption of our lemma,  $Cov(d^TX,b^TH)=0$ for all these $b$. Denote $V_0=(v_1,v_2..v_k)$. Moreover, since $u_i, v_j$ are canonical directions, $Cov(u_i^TX,v_j^TH)=0$ if $i \ne j$. Therefore 
\begin{equation*}
0=Cov((\sum_{i=1}^k\alpha_iu_i)^TX,v_j^TH)=Cov((\alpha_ju_j)^TX,v_j^TH)
\end{equation*}
for all $j$. This implies $\alpha_j=0$ for $j=1..k$ since $Cor(u_j^TX,v_j^TH)$ is the $j^{th}$ canonical correlation which is non zero. Therefore $c=0$, $a=d$ lies in the column space of $U_2$.

\end{proof}
The above lemma shows that in order to find the column space of $U_2$, it suffices to find $2k$ linear independent directions that satisfies lemma 2, which is easy. Run a CCA between random vectors $X^1$ and $(X^2;X^3)$, we have:
\begin{lemma}
the last $k$ canonical directions of $(X^2;X^3)$ has 0 correlation matrix with $H$, hence satisfy lemma 4. 
\end{lemma}
\begin{proof}
Denote the rotation matrix corresponding to the last $k$ directions by $R_1\in R^{2k\times k}$, $X_{23}=R_1^T(X_2;X_3)$ are the last k canonical variables. By assumption 2, $\mathbb{E}[X_{23}|H]=M_5H$ and $\mathbb{E}[X^1|H]=M_1H$. Lemma 1 indicates $\Sigma_{X_{23},1}=M_5M_1^T=0$, since $M_1$ is $k*k$ full rank by assumption 3, $M_5=0$, so $\Sigma_{X_{23},H}=\mathbb{E}[M_5H]=0$.
\end{proof}
Similarly, run a CCA (or Canonical Covariance Analysis) between random vectors $X^3$ and $(X^1;X^2)$, the last $k$ canonical directions of $(X^1;X^2)$ has 0 correlation matrix with $H$, hence satisfy lemma 2. Denote the rotation matrix corresponding to the last $k$ directions by $R_2\in R^{2k\times k}$. For notation convenience, let
\begin{displaymath} 
\mathbf{R_1} = \left( \begin{array}{c} R_{11} \\
R_{21}\\ 
\end{array} \right) 
\end{displaymath}

\begin{displaymath} 
\mathbf{R_2} = \left( \begin{array}{c} R_{12} \\
R_{22}\\  
\end{array} \right) 
\end{displaymath}
where all the blocks $R_{i,j}$ are $k\times k$.
Finally, let $O$ be $k\times k$ matrix with all zeros, Let 
\begin{equation} 
\mathbf{R} = \left( \begin{array}{cc} R_{11}& O \\
R_{21}& R_{12}\\
O & R_{22}\\ 
\end{array} \right) 
\end{equation}
If the $R$ is full rank (which is true in most case), the column space of $R$ is exactly the column space of $U_2$ since every column of $R$ satisfies lemma 2, and it form a basis.\\

Based on the above argument, the algorithm for finding the optimal $k$ dimensional subspace is:
\begin{table}[!hbp]
\begin{center}
\begin{tabular}{|l|l|l|}
\hline
&\textbf{Algorithm: Optimal Weighting of Three Views}\\
\hline
Step1 & Estimate the $3k\times 3k$ covariance matrix $\Sigma_{X,X}$ empirically, \\&and compute $Q$ as the square root of $\Sigma_{X,X}$ \\
\hline
Step2 & Perform CCA between $X^1$ and$(X^2,X^3)$ to obtain rotation matrix $R_1$.\\& Perform CCA between $X^3$ and$(X^1,X^2)$ to obtain rotation matrix $R_2$\\
\hline
Step3 & Construct $R$ based on $R_1, R_2$ with equation (4) \\
\hline
Step4  & Compute $P_2=QR$\\
\hline
Step5  & Compute $P_1$ by finding the orthogonal complement of $P_2$\\
\hline
Step6  & Compute $U_1=Q^{-1}P_1$\\&$U_1$ is the matrix which project $X$ to the optimal $k$ dimensional subspace.\\
\hline

\end{tabular}
\end{center}
\caption{Finding Optimal $k$ Dimensional Subspace with Three Views}
\end{table}     
\begin{remark}
In dimension reduction point of view, running two views CCA between each pair of views reduce the dimension from $d_1+d_2+d_3$ to $3k$ and running the three views algorithm reduce the dimension from $3k$ to $k$. By doing CCA we find a $k$ dimensional subspace in the $d_1+d_2+d_3$ huge space which contains all the useful information in predicting the hidden state $H$ and hence the variable $Y$. In fact this is the optimal unsupervised dimension reduction possible since the projection (in the Hilbert Space of random variables) of the hidden state onto the $d_1+d_2+d_3$ feature space is exactly the $k$ dimensions given by the CCA. 
\end{remark}

\section{Experiments On Simulated Data}
In this section the three view algorithm is applied to a normal model. In this model, we have a $k=10$ dimensional normal hidden state $H$ with mean 0 and identity covariance. Conditional on $H$, three views $X^i$ has normal distribution with mean $A_iZ$ ($A_i\in k\times k$) and covariance $\sigma_i I$ ($\sigma_1=2$, $\sigma_2=0.5$, $\sigma_3=0.2$). Our goal is to predict a random variable $Y$. Conditioning on $Z$, $y$ is a normal with mean $\beta Z$ ($\beta\in 1\times k$) and variance $\sigma=0.5$ ($A_i$ and $\beta $ are generated at random).\\

In the first experiment, we compare three groups of features. The first group is all the three views $X=(X^1,X^2,X^3)$ (denoted as $S_1$). The second is the $k$ dimensional feature $U_1^TX$ obtained by our algorithm (denoted as $S_2$). The third is also a $k$ dimensional feature, but it's just averaging three views, i.e. $X^1+X^2+X^3$ (denoted as $S_3$). We want to compare the square loss of the optimal predictor with the three features, therefore we run a regression with large amount of labeled data (5000) to make sure our linear predictors converges to the optimal ones. This experiment is repeated 100 times (use 100 different rotation matrix $A_i$s).\\

Figure 3 shows the square loss of the optimal predictor of $Y$ with three groups of features. The Y axis is the square loss while the X axis indicates different trials. The left of Figure 3 shows the square loss of $S_1$ and $S_2$ ($S_2$ is learned with 50000 unlabeled data), the right side of Figure 3 shows the square loss of $S_2$ and $S_3$. Easy to see the square loss of $S_1$ and $S_2$ is pretty close most of the time while the square loss of $S_3$ is much larger. Figure 4 shows the histogram of optimal square loss ratio for this 100 trials. The left figure is $\frac{\text{square loss of } S_2}{\text{square loss of } S_1}$ and the right figure is $\frac{\text{square loss of } S_3}{\text{square loss of } S_1}$. Easy to see in most cases $\frac{\text{square loss of } S_2}{\text{square loss of } S_1}$ distributed very close to $1$, i.e. the optimal square loss of $S_1$ and $S_2$ are almost the same while in most cases optimal square loss of $S_3$ is way larger than $S_1$.

\begin{figure}[h]
\centering
\includegraphics[width=11cm]{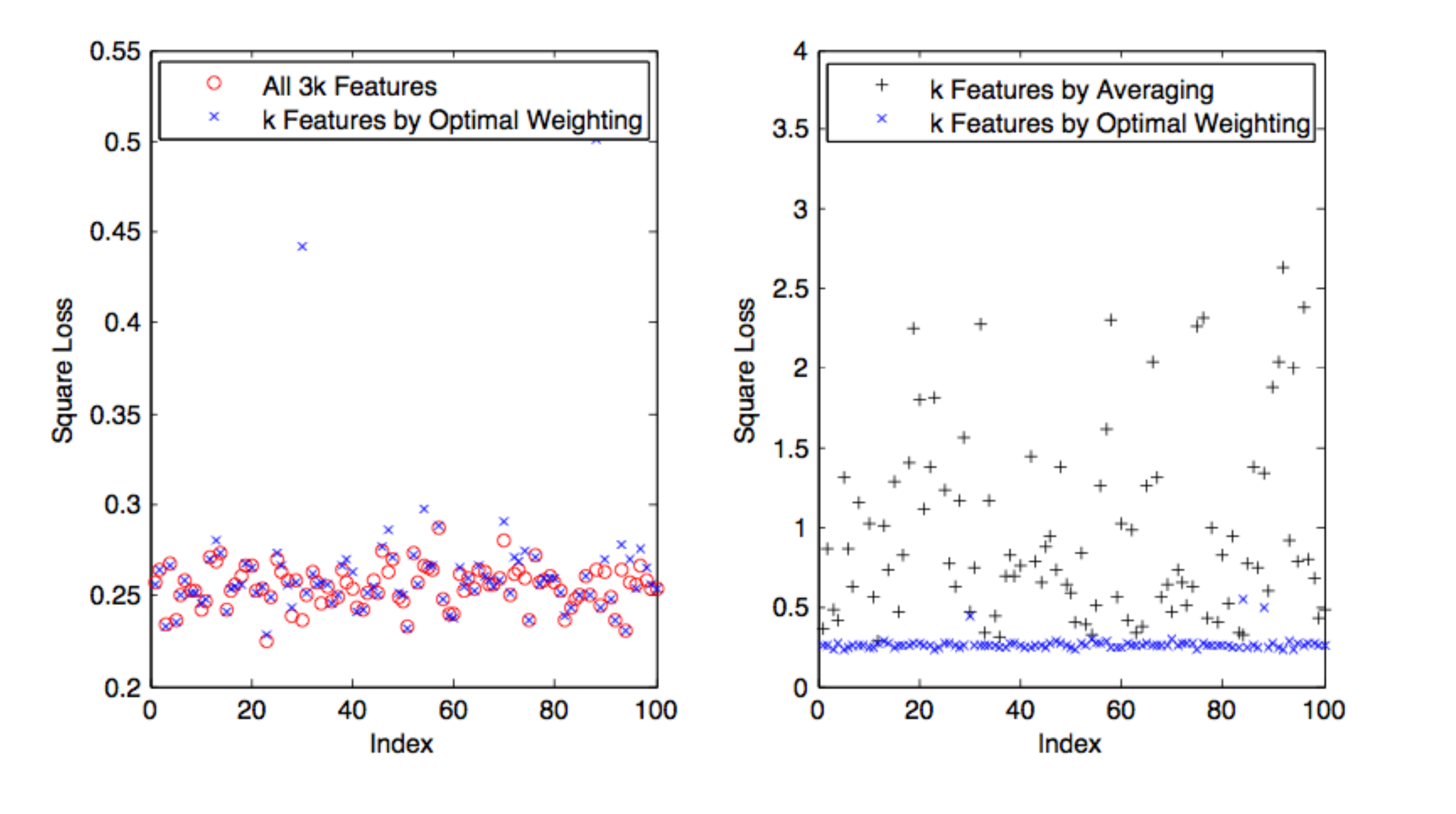}
\label{f3:non:float}
\caption{The square loss of the optimal linear predictor using three different feature sets}
\end{figure}

\begin{figure}[h]
\centering
\includegraphics[width=11cm]{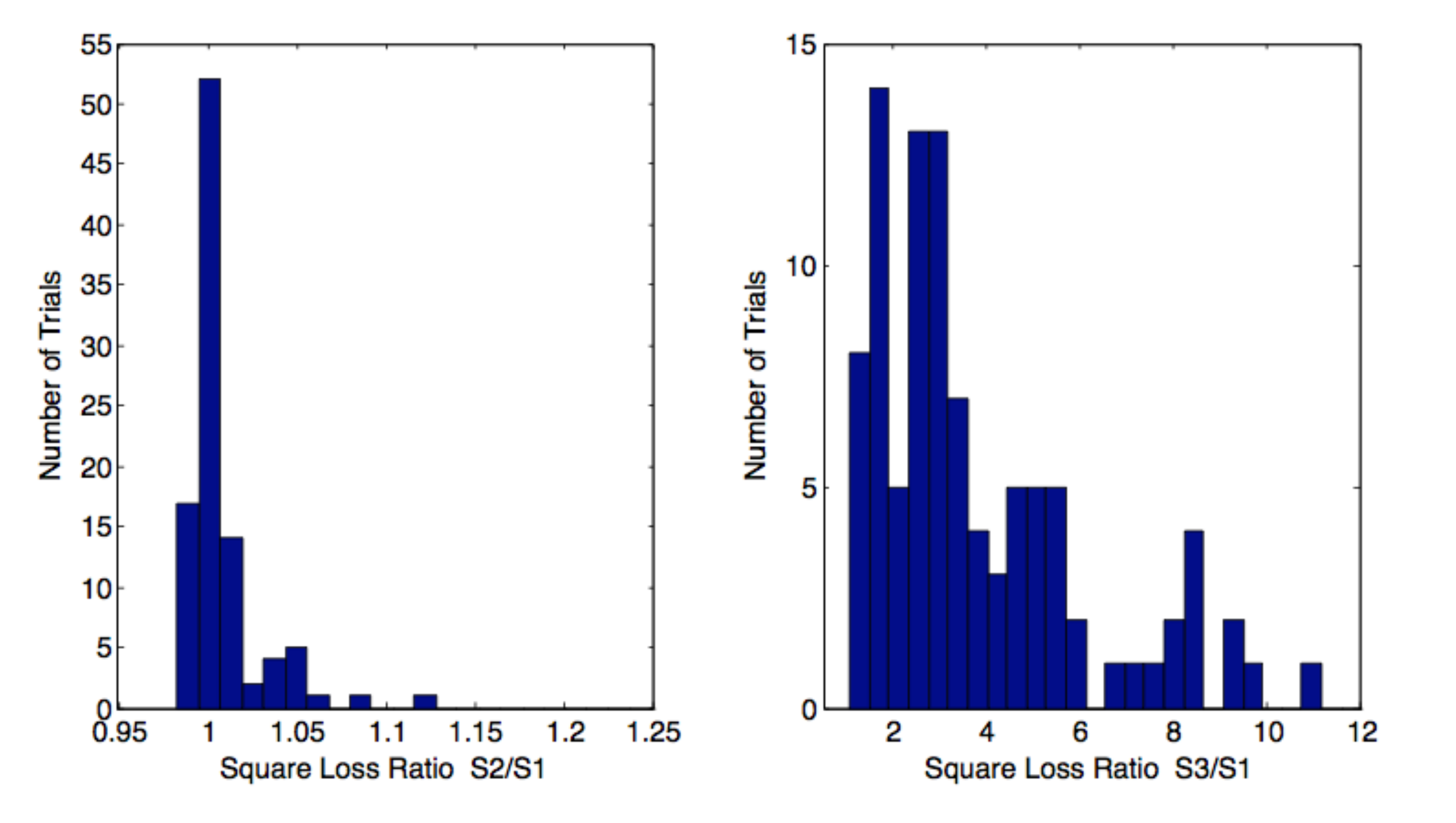}
\label{f3:non:float}
\caption{The histogram of optimal square loss ratio of different feature sets}
\end{figure}

The second experiment is about the sample size. We run the three views algorithm on different amount of $X$ to obtain $S_2$ (The sample size of Group 1 to 7 are: 500, 1000,  2000, 4000, 8000, 10000, 20000). For each group, we run 100 experiments and box plot the optimal square loss of $S_2$ in for each group. \\

Figure 5 shows the optimal square loss of $S_2$ of different sample sizes. The dash line at about y=0.256 is the average optimal square loss of the $3k$ feature set $S_1$, i.e. the asymptote optimal if the sample size is large enough. Our algorithm performs better as sample size increases. When sample size is about $20000$ (Group 7) the square loss of $S_2$ becomes close to the square loss of $S_1$.\\

\begin{figure}[h]
\centering
\includegraphics[width=8.6cm]{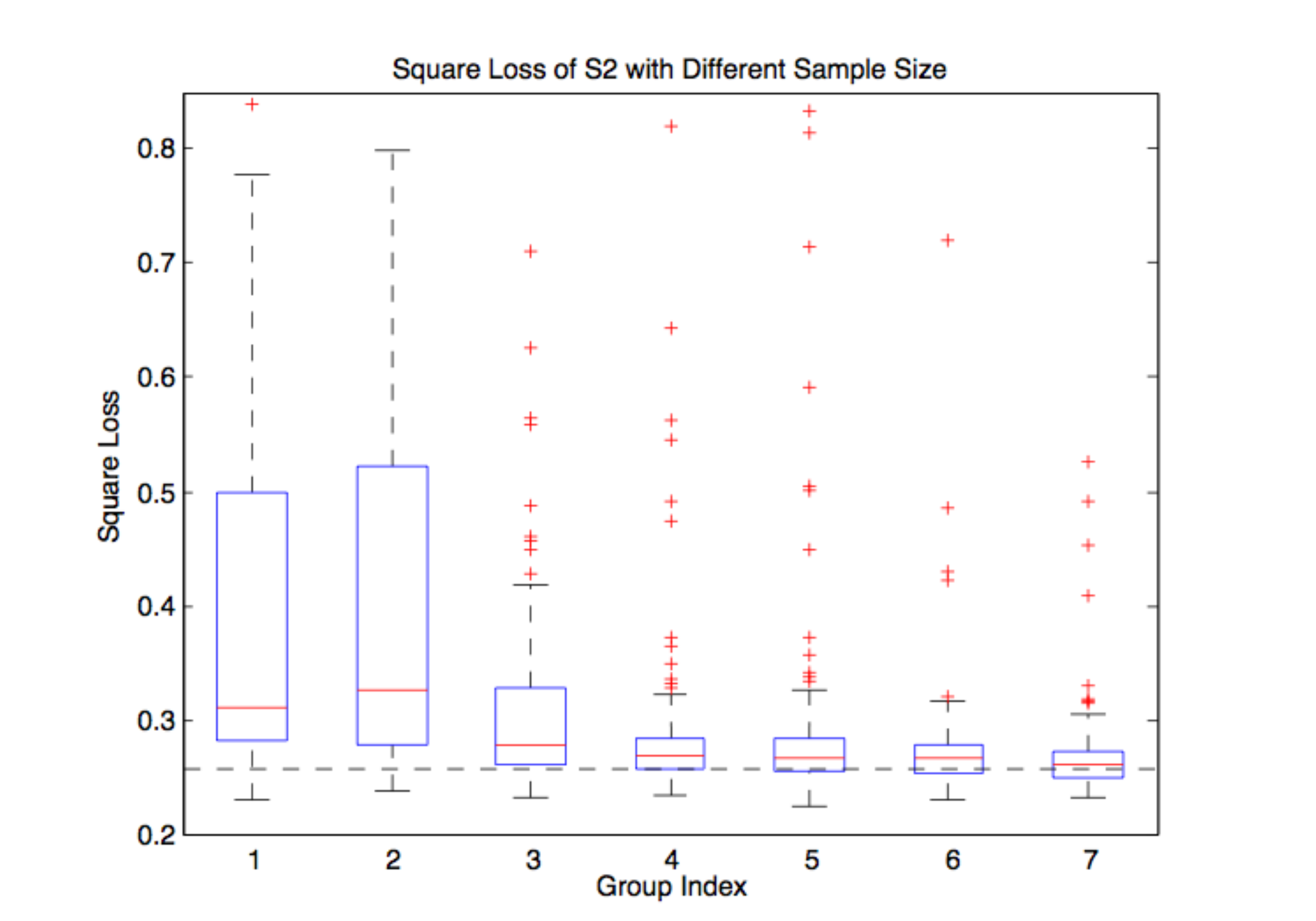} 
\caption{Box plot of the optimal square loss of $S_2$ obtained by three view algorithm with different sample size. The sample size of Group 1 to 7 are: 500, 1000,  2000, 4000, 8000, 10000, 20000. The dash line at about y=0.256 is the average optimal square loss of the $3k$ feature set $S_1$, i.e. the asymptote if the sample size is large enough.}
\end{figure}

The third experiment shows the advantage of our three view algorithm when the amount of labeled data is limited. Still consider predicting $Y$ with linear regression. As we know, the square loss of regression can be decomposed into bias and variance.In section 3 and the first experiment it is shown that the dimension reduction of our three views algorithm doesn't introduce any bias. Moreover, the variances are reduced due to reduced dimensionality. In the third experiment,  we compare the square loss of predicting $Y$ with $S_1$ and $S_2$ ($S_2$ is learned with 50000 unlabeled data). Four groups of experiment are performed with different amounts of labeled data (labeled data size are: 40,80,150,400, the dimension of $S_1$ is 30 and the dimension of $S_2$ is 10). In each group, 25 different model parameters (different $A_i$ and $\beta$) are randomly generated and for each parameter set up, we estimate the square loss by simulation. \\

The square loss of  25 parameter set ups in each group are box plotted in Figure 6 (labeled data size increase from left to right). Easy to see, when lacking labeled data our three view feature $S_2$ outperform the original feature $S_1$ and the difference becomes smaller when more labeled data are available.\\
\begin{figure}[h]
\centering
\includegraphics[width=8.6cm]{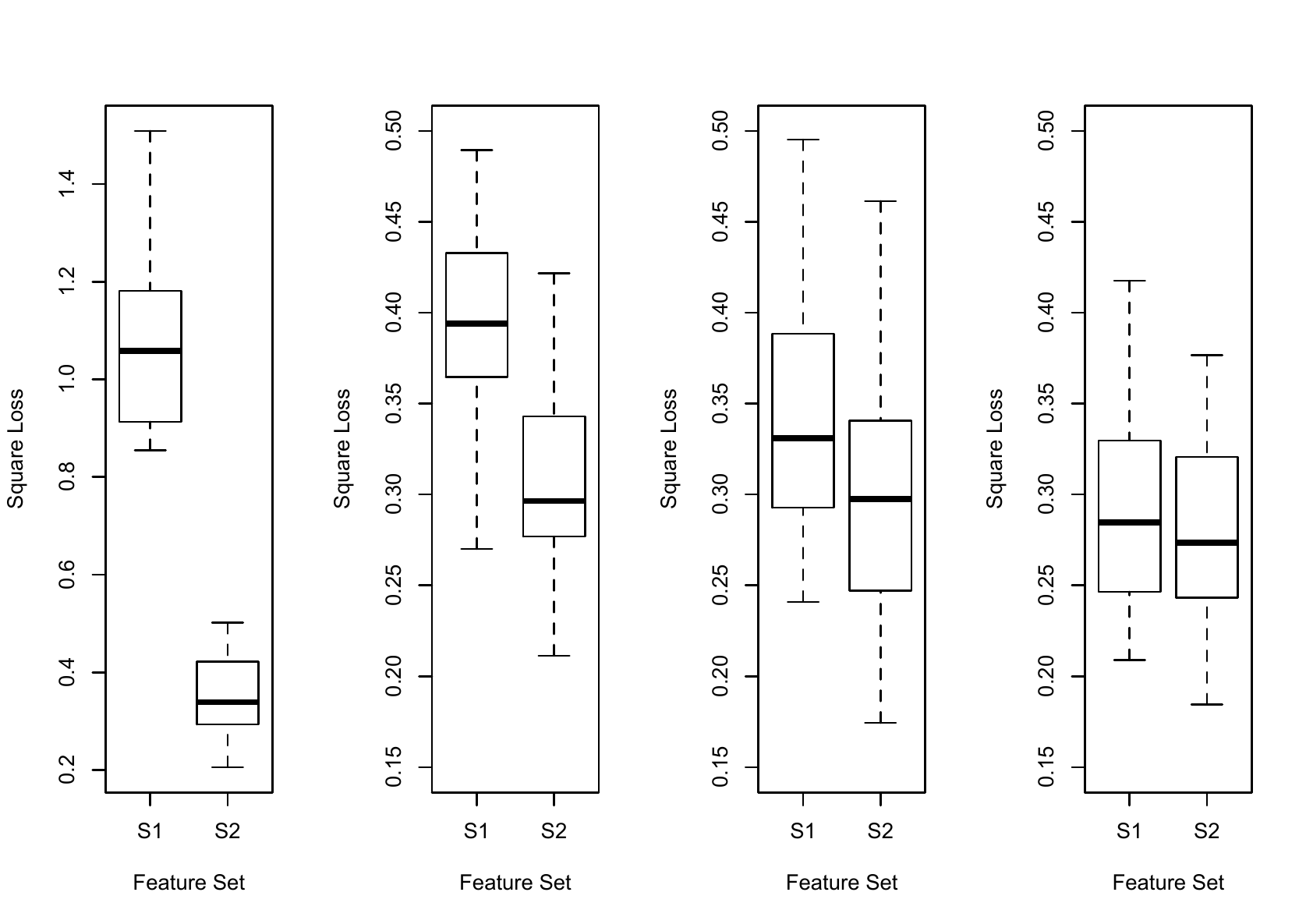} 
\caption{Box plot of the square loss of predicting $Y$ with $S_1$ and $S_2$ when labeled data is limited. The number of labeled data of each figure are (from left to right): 40,80,150,400.}
\end{figure}
\section{Summary}
We see how CCA can be applied for dimension reduction and optimal weighting in the multi-view model with a hidden state, which is assumed to carry most information for supervised learning problems. After doing CCA, we end up with a $k$ dimensional feature space which achieves optimal dimension reduction. This dimension reduction method works very well when huge amount of unlabeled data are available while labeled data are limited. If more than three views are available, we only need to group the views into three disjoint parts and these three parts can act as three views in our algorithm.

\bibliography{three_view_formal}{}
\bibliographystyle{plain}
\end{document}